\documentclass{article}

\usepackage{microtype}
\usepackage{graphicx}
\usepackage{subcaption}
\usepackage{booktabs} 

\usepackage{hyperref}

\usepackage[preprint]{icml2026}

\usepackage{amsmath}
\usepackage{amssymb}
\usepackage{mathtools}
\usepackage{amsthm}

\usepackage{Definitions}
\usepackage[capitalize,noabbrev]{cleveref}

\theoremstyle{plain}
\newtheorem{theorem}{Theorem}[section]
\newtheorem{proposition}[theorem]{Proposition}

\newtheorem{corollary}[theorem]{Corollary}
\theoremstyle{definition}
\newtheorem{definition}[theorem]{Definition}
\newtheorem{assumption}[theorem]{Assumption}
\theoremstyle{remark}

\usepackage[textsize=tiny]{todonotes}

\usepackage{xcolor} 

\icmltitlerunning{Generating Robust Portfolios of Optimization Models using Large Language Models}

\begin{document}

\twocolumn[
  \icmltitle{Generating Robust Portfolios of Optimization Models \\ using Large Language Models}



  \icmlsetsymbol{equal}{*}

  \begin{icmlauthorlist}
    \icmlauthor{Eleni Straitouri}{mpi}
    \icmlauthor{Cheol Woo Kim}{harvard}
    \icmlauthor{Milind Tambe}{harvard}
  \end{icmlauthorlist}

  \icmlaffiliation{mpi}{Max Planck Institute for Software Systems, Kaiserslautern, Germany}
  \icmlaffiliation{harvard}{Harvard University, Cambridge, United States}

  \icmlcorrespondingauthor{Eleni Straitouri}{estraitouri@mpi-sws.org}

  \icmlkeywords{Optimization Modeling, Portfolios, LLMs}

  \vskip 0.3in
]



\printAffiliationsAndNotice{}  

\begin{abstract}

Mathematical optimization is a powerful tool for structured decision-making across domains such as resource allocation and planning. Formulating optimization models faithful to reality, though, remains a significant bottleneck as it typically demands both domain expertise and optimization knowledge that are often scarce. Recent advances in large language models (LLMs) promise to bridge this gap, enabling the generation of candidate optimization models from natural language descriptions. However, there is no guarantee that any single LLM-generated model is reliable, and existing approaches that output only one model are therefore risky.
In this work, we propose a novel algorithm that generates a portfolio of optimization models, designed to be robust to the limitations of LLMs. Our method exploits the observation that a single LLM can play two distinct roles---as a stochastic generator and as a reasoning evaluator---and proposes a unified framework that leverages both capabilities in a complementary manner. We provide theoretical guarantees showing that, as long as either the generator or the evaluator is well-aligned with human preferences, the portfolio is guaranteed to contain high-quality candidates, enabling a principled human-in-the-loop process in which a decision-maker can review multiple candidates before committing to one. We further validate our approach empirically, demonstrating strong performance across a range of optimization modeling tasks.

\end{abstract}

\section{Introduction}
Formalizing a structured decision-making task as a mathematical optimization problem plays a pivotal role in designing optimal decision policies in domains such as resource allocation and planning~\citep{brill1979use, katoh1998resource, vercellis2011business}. 
However, defining an optimization model that accurately reflects all the real-world     requirements and constraints of the decision-making task can be rather challenging; 
typically, it requires not only exhaustive manual tuning, but also a combination of domain expertise and deep optimization knowledge. 
%
%
%
%
%

To address this challenge, there has been a growing interest in leveraging Large Language Models (LLMs) to automate the definition of optimization models given a task description in natural language.
Recent lines of work typically focus on either automating the entire model definition~\citep{yang2023large, yang2024optibench, ahmaditeshnizi24optimus, astorga2024autoformulation, zhang2024solving, jiang2024llmopt,ahmed2024lm4opt, huang2025orlm, xiao2025survey} or a partial model definition limited to the design of the objective (reward) function  
in 
~\citep{icarte2022reward,yu2023language,shinn2023reflexion,ma2024eureka,hwang2024promptable,xie2024textreward,behari2024decision, verma2025balancing, sun2025large}.
%

Prior work, though, typically proposes computationally-heavy approaches, requiring additional training or fine-tuning of language models, with the objective to generate a single optimization model, while lacking guarantees on its quality. 
In this work, we introduce a lightweight, training-free algorithm to generate a portfolio of optimization models with guarantees over the quality of generated models while being robust to 
limitations of language models.
To achieve this, our algorithm leverages a dual perspective on the capabilities of language models that we describe next.

Language models can operate as stochastic \textit{generators}~\citep{verma2025balancing, cardenoso2025leveraging} that can provide diverse models through repeated stochastic sampling, accounting for different trade-offs present in the optimization task at hand. 
%
%
In addition, language models, can also operate as judges or reasoning evaluators~\citep{verma2025balancing, saccon2025automated} arguing about the quality of given inputs, based on their world knowledge and reasoning capabilities.
%
By unifying these modes of operation, our approach leverages their complementary strengths to provide robustness guarantees over the quality of models comprising our portfolios, 
%
%
allowing a decision-maker to review multiple high quality candidates before committing to
one.  

%
%



\xhdr{Contributions} We propose a method that first uses 
a language model, namely a \textit{generator},  to generate a distribution of candidate optimization models through repeated stochastic sampling. 
Next, our method uses a judge agent, namely an \textit{evaluator} to rank these optimization models based on how well they align with the optimization task description in natural language.
We construct a portfolio comprising the candidate optimization models with the highest evaluator ranks that have a total generation probability above a user-specified threshold. 
We show that in this way our portfolio guarantees to include high quality candidates   if either of the generator or the evaluator satisfy a certain level of alignment with respect to human preferences. We empirically verify the strong performance of our portfolio in experiments with both synthetic and real data.


\section{Building a Portfolio through a Generator and Evaluator}
\label{sec:portfolio}
Let $d$ be the natural language description of an optimization problem,
$g$ be a stochastic generator model and $e$ an evaluator model.

\xhdr{Generator} %
Given a natural language description $d\in \Dcal$, where $\Dcal$ is the space of any natural language prompt, we consider the generator $g$ as a probability distribution $p$ over the space of candidate optimization models $\Ocal$.%
~\footnote{We assume that $\Ocal = \Ocal_v \cup \bot$, where $\Ocal_v$ is the space of text representations of any valid optimization model, and $\bot = \Dcal \setminus \Ocal_{v}$. Note that $\Ocal$ can be finite or infinite.}  
%
%


\xhdr{Evaluator} Given the optimization problem description $d$, we consider the evaluator as a ranking policy $\pi_{e}$  over the space of candidate optimization models $o \in \Ocal$, inducing a ranking 
\begin{equation}
    \pi_{e}(d) = (o_{(1)^e}, o_{(2)^e}, \dots),
\end{equation}
where the subscript ${(\cdot)}$ denotes that rank of candidate $o$. 
We assume that the lower the rank, the better the candidate given the description $d$, according to world knowledge of the evaluator, and ties are broken randomly.

We use the probability distribution $p$ induced by the generator values $g(d)$ and the ranking $\pi_{e}(d)$ induced by the evaluator $e$ to construct our portfolio $\Pcal$ of candidate optimization models  as follows 
\begin{align}~\label{eq:portfolio}
    \Pcal(d; \alpha) = \left \{  o_{(i)^e}  \right \}_{i=1}^{k^*(\alpha)},    
\end{align}
where $\alpha \in (0,1)$ is a user-defined parameter and 
\begin{align*}~\label{eq:portfolio-size}
    k^*(\alpha) = \inf \left \{ k \in \NN : \sum_{i=1}^{k} p(o_{(i)^e}) \geq 1 - \alpha \right \}.
\end{align*}

By constructing our portfolio in this way, we make sure that, for a small enough $\alpha$ our portfolio will include candidate models either because they have a low (good) evaluator rank or because they have a high probability of being generated
%
This means that our portfolio will include high-quality models as long as either the evaluator ranks them low, or the generator produces them with high probability, \ie, whenever, either of them is well-aligned with human preferences. 
Motivated by this observation, in what follows, we prove that our portfolio guarantees to include high-quality candidates if either of the generator or evaluator are human-aligned.

\section{Robust Portfolio Generation}
\label{sec:robustness}
We begin by defining human-alignment of the generator and evaluator under the following assumption
\begin{assumption}
Given the natural description $d$ there is a human ranking policy $\pi^*$, inducing a ranking $\pi^{*}(d)$ over $o \in \Ocal$, with
\begin{align}
    \pi^{*}(d) =( o_{(1)^{*}} , o_{(2)^{*}} , \dots ),
\end{align}
where the lower the rank, the higher the quality of the candidate optimization model $o$ according to human preferences.   
\end{assumption}

Based on the above, we provide the following definitions.
\begin{definition}[Evaluator Alignment]~\label{def:evaluator-alignment} An evaluator is human-aligned given $d$, if the induced ranking 
\begin{equation*}
    \pi_{e}(d) = \pi^{*}(d).
\end{equation*}
\end{definition}
\begin{definition}[Generator Alignment]~\label{def:generator-alignment} 
    A generator model inducing a probability $p$ over candidate optimization models $o \in \Ocal$ is human-aligned given $d$, if for any $o_{(i)^{*}}, o_{(j)^{*}} \in \pi^{*}(d)$ with $i \leq j$, it holds 
    \begin{equation}
        p(o_{(i)^{*}}) \geq p(o_{(j)^{*}}).
    \end{equation}
\end{definition}
Intuitively, a human-aligned generator generates candidates judged as
high quality according to human preferences with high probability.   

We show that under either evaluator or generator alignment our portfolio is guaranteed to include high quality candidates or more formally achieve positive \textit{coverage}, where we define coverage as follows.
\begin{definition}[Portfolio coverage]
    A  portfolio $\Pcal$ of $k$ candidates has coverage
    \begin{align}
        c(\Pcal) = \frac{\sum_{i=1}^k \II\{o_{(i)^*} \in \Pcal\}}{k}
    \end{align}
\end{definition}

Under evaluator alignment and any generator, we prove:~\footnote{Proofs are deferred to Appendix~\ref{app:proofs}.}  

\begin{corollary}[Evaluator robustness]~\label{cor:eval-robustness}
    Assume an aligned evaluator under Definition~\ref{def:evaluator-alignment} for a description $d$. For any value of $\alpha \in (0,1)$ and any generator, the portfolio $\Pcal(d; \alpha)$ constructed using Eq.~\ref{eq:portfolio} has coverage $c(\Pcal(d ; \alpha)) = 1$. 
\end{corollary}

Under generator alignment and any evaluator, we prove

\begin{proposition}~\label{prop:coverage-bound}
    Assume an aligned generator under Definition~\ref{def:generator-alignment} for a description $d$. For any value of $\alpha \in (0,1/2)$ and any evaluator, any non-empty portfolio $\Pcal(d ; \alpha)$ constructed using Eq.~\ref{eq:portfolio}  has coverage
    \begin{align*}
        c(\Pcal(d ; \alpha)) > \frac{1 - 2\alpha}{k^{*}(\alpha)} > 0.
    \end{align*}
\end{proposition}

\section{Experiments}
\label{sec:experiments}
We begin by simulating the generator and evaluator under different levels of human-alignment in a synthetic data setup, where we i) empirically verify the robustness of our %
portfolio, and ii) investigate the effect of human-alignment on %
the portfolio coverage and size. 
We proceed with a real data implementation of our portfolios on optimization modeling, where we show that the optimization models comprising our portfolios are qualitatively superior compared to optimization models provided by randomly sampled portfolios.  

\subsection{Simulated Portfolios}
\label{subsec:simulation}
In our synthetic data setup, we consider a finite space of candidate optimization models 
$\Ocal$, such that $|\Ocal| = K$, for $K \in \{10, 20, 50 , 100\}$, with human ranking $(1, 2, ..., K)$.
%
We simulate several generator and evaluator types, each with a different level of human alignment as described below.   
For each  generator-evaluator pair, we use Eq.~\ref{eq:portfolio} to construct a portfolio for each $\alpha$ value from $0$ to $1$ with step $0.02$ and repeat the experiment with $40$ different seeds. 

\xhdr{Generators} We implement the following generator types (refer to Appendix~\ref{app:deatils-simulation} for more details).
\begin{itemize}
\item \textsc{Aligned}. A generator satisfying Definition~\ref{def:generator-alignment}.
\item \textsc{Weakly Aligned}.  A generator violating  Definition~\ref{def:generator-alignment}, for which though   Proposition~\ref{prop:coverage-bound} holds for any portfolio of $K/2$ generated candidates.
\item \textsc{Uniform}. A generator where each generated candidate $i \in [K]$ has probability $p(i) = 1/K$.
\item \textsc{Misaligned}. A reversely aligned generator. 
\end{itemize}

\xhdr{Evaluators} We characterize each evaluator by the fraction of incorrectly ranked candidates, which we denote as the evaluator error $\epsilon = \frac{\sum_{i=1}^K \II\{o_{(i)^{*} }\neq o_{(i)^{e}} \}}{K}$. We implement the evaluator for $\epsilon \in \{0, 0.3, 0.5, 0.7, 1\}$, where $\epsilon = 0$ characterizes a human-aligned evaluator with $\pi^{*} = \pi_{e}$, and $\epsilon = 1$, an evaluator $e$ such that $\forall i \in [K], o_{(i)^e} = o_{(K + 1 - i)^*}$.

\begin{figure}[t]
    \centering
    \includegraphics[width=.75\linewidth]{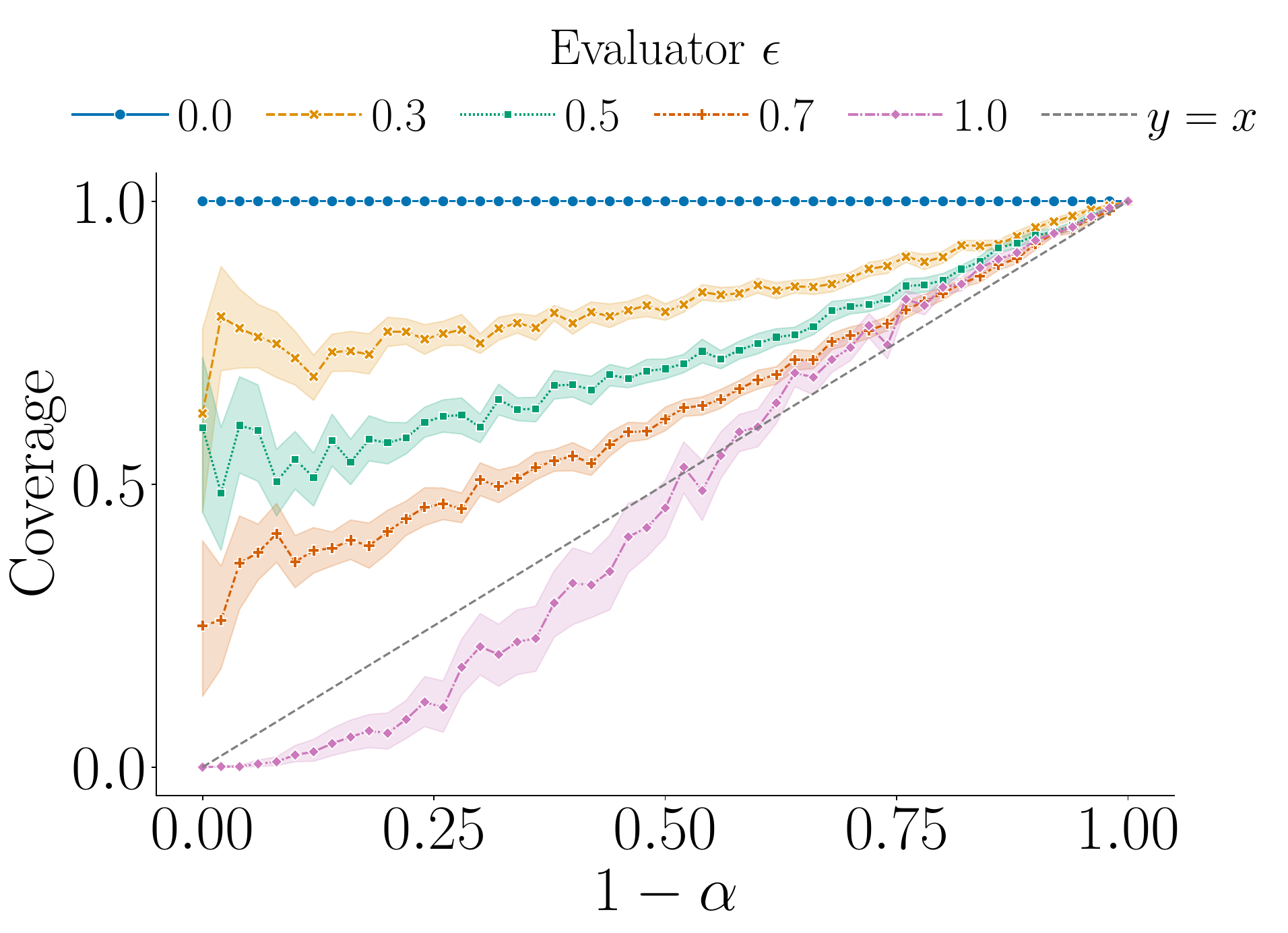}
    \caption{Portfolio mean coverage against the value of $1-\alpha$ for the \textsc{Weakly Aligned} generator paired with each evaluator for $K=100$. The mean is over $40$ iterations and shaded areas represent $95\%$ confidence intervals.}
    \label{fig:weakly-aligned-coverage-K100}
\end{figure}
\begin{figure}[t]
    \centering
    \includegraphics[width=.725\linewidth]
{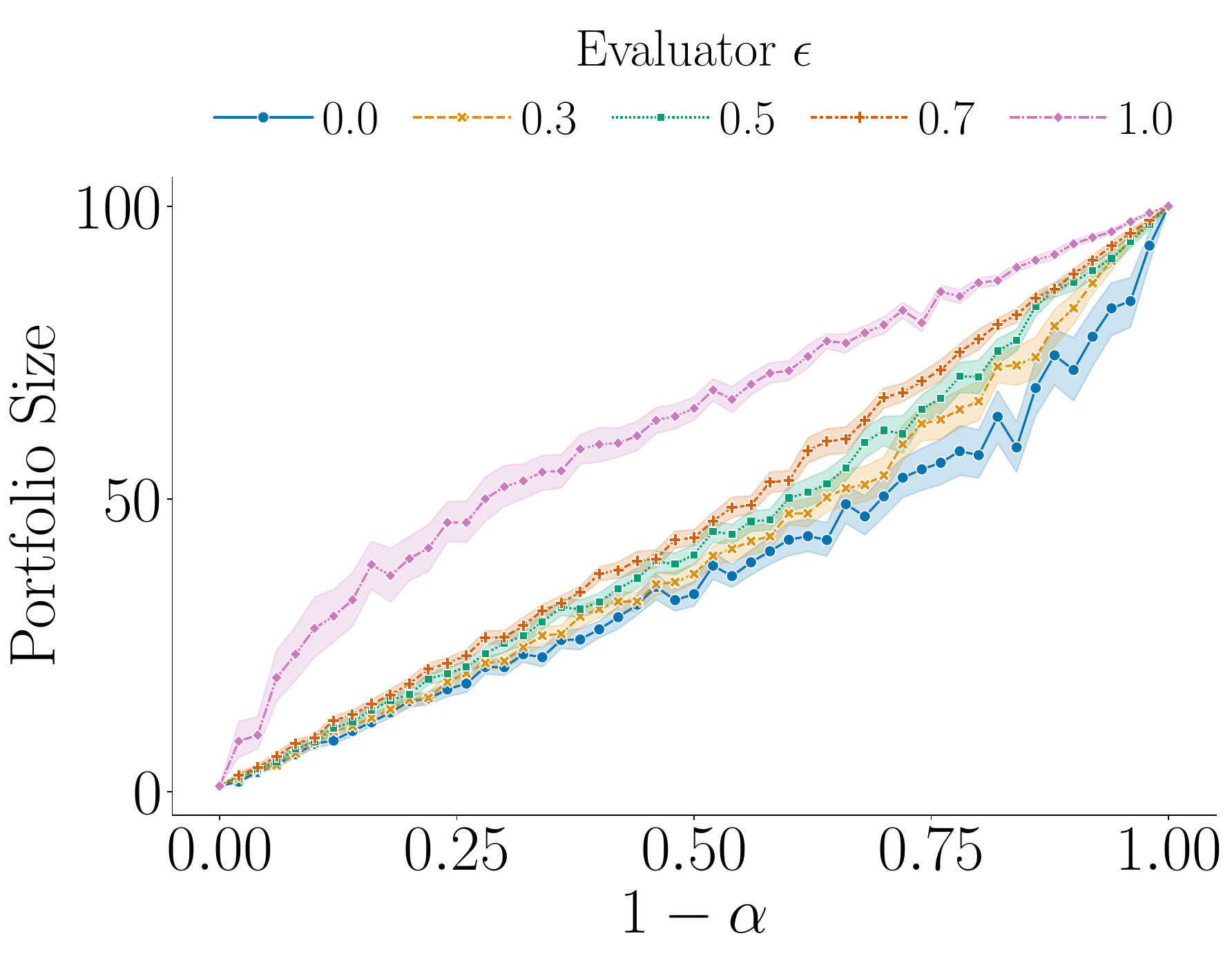}
    \caption{Portfolio mean size against the value of $1-\alpha$ for the \textsc{Weakly Aligned} generator paired with each evaluator for $K=100$. The mean is over $40$ iterations and shaded areas represent $95\%$ confidence intervals.}
    \label{fig:weakly-aligned-size-K100}
\end{figure}
\begin{figure}[h!]
    \centering
    \includegraphics[width=.8\linewidth]
{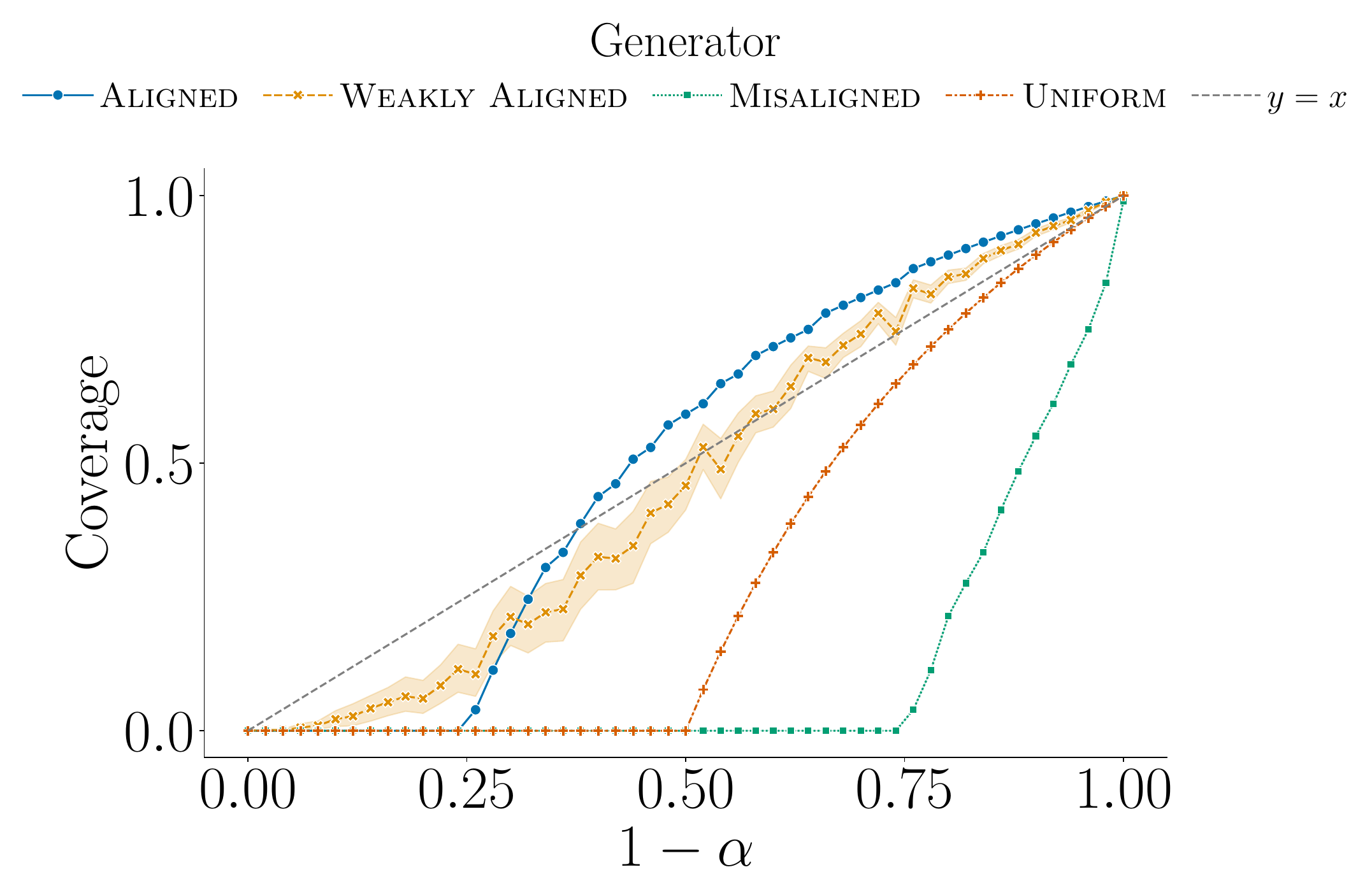}
    \caption{Portfolio mean coverage against the value of $1-\alpha$ for the evaluator with $\epsilon=1.0$ paired with each generator  for $K=100$. The mean is over $40$ iterations and shaded areas represent $95\%$ confidence intervals.}
    \label{fig:e-1.0-coverage-K100}
\end{figure}
\begin{figure}[th!]
    \centering
    \includegraphics[width=.76\linewidth]
    {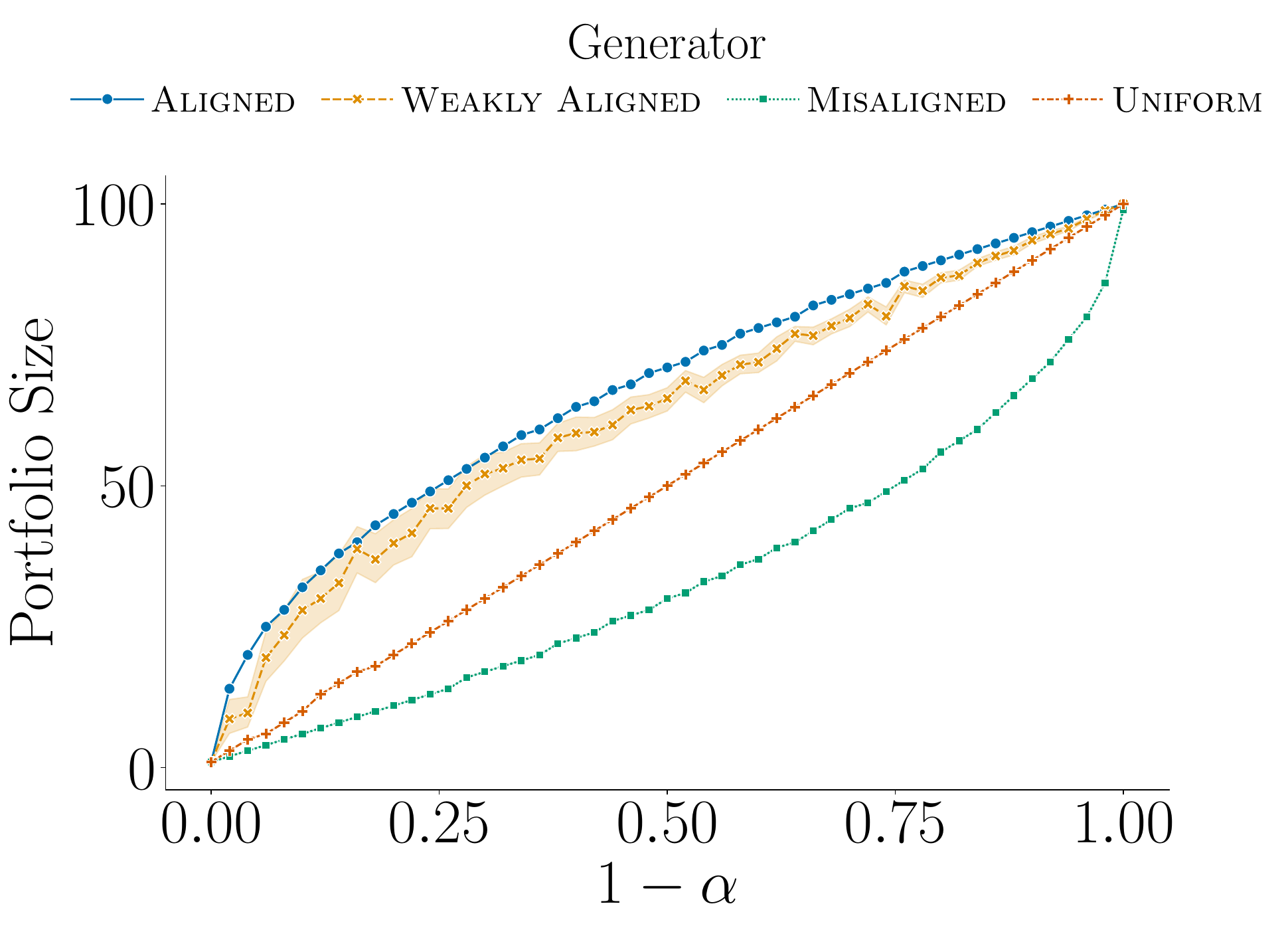}
    \caption{Portfolio mean size against the value of $1-\alpha$ for the evaluator with $\epsilon=1.0$ paired with each generator  for $K=100$. The mean is over $40$ iterations and shaded areas represent $95\%$ confidence intervals.}
    \label{fig:e-1.0-size-K100}
\end{figure}
\vspace{-1mm}

\xhdr{Results} Figure~\ref{fig:weakly-aligned-coverage-K100} and~\ref{fig:e-1.0-coverage-K100} show---in consistence with Proposition~\ref{prop:coverage-bound}--- that for $\alpha < 0.5$ our portfolio achieves  positive coverage even under the \textsc{Weakly Aligned} generator independently of the evaluator error $\epsilon$. In fact, Figure~\ref{fig:weakly-aligned-coverage-K100} shows, in practice, that coverage is lower bounded by $1-\alpha$---a tighter lower bound compared to Proposition~\ref{prop:coverage-bound}---as coverage is always above the diagonal for $\alpha < 0.5$. Further, Figure~\ref{fig:weakly-aligned-coverage-K100}--~\ref{fig:e-1.0-size-K100} reveal notable insights on how the level of human alignment controls the trade-off between portfolio coverage and its size; Figure~\ref{fig:weakly-aligned-coverage-K100} and~\ref{fig:weakly-aligned-size-K100} show that under a \textsc{Weakly Aligned} generator, the lower the evaluator error, the higher the coverage and the lower the portfolio size for the same value of $\alpha$; Figure~\ref{fig:e-1.0-coverage-K100} and~\ref{fig:e-1.0-size-K100} show that under the worst evaluator, the more aligned the generator, the higher the coverage, at the price though of larger portfolios.

\subsection{Portfolios of Optimization Models}
\label{subsec:real}
\begin{figure}[t]
    \centering
   \subfloat[LLM Evaluator]{\includegraphics[width=.7\linewidth]
    {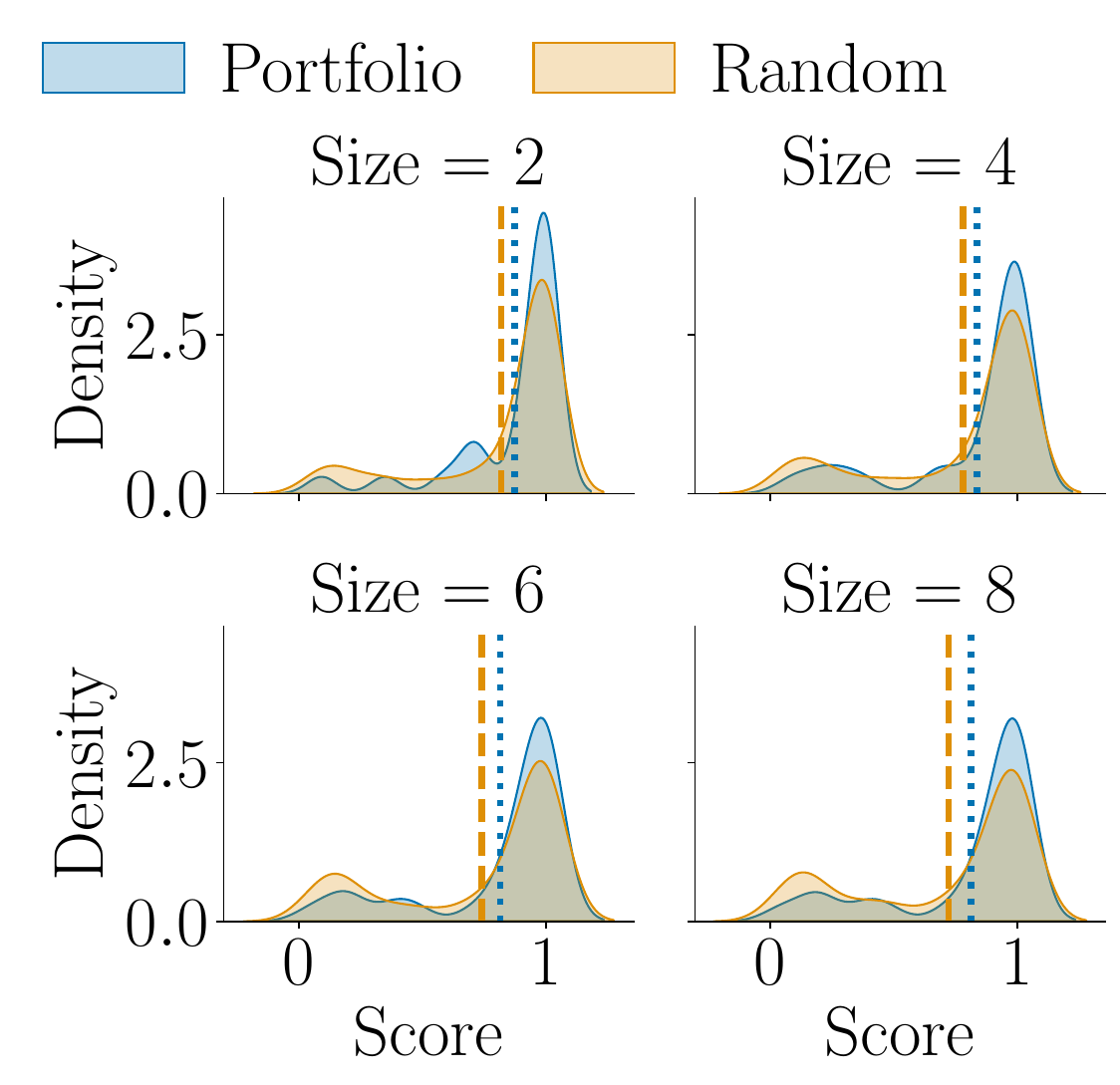}}\\
    \subfloat[Evaluator using the Generator Probabilities]{\includegraphics[width=.7\linewidth]{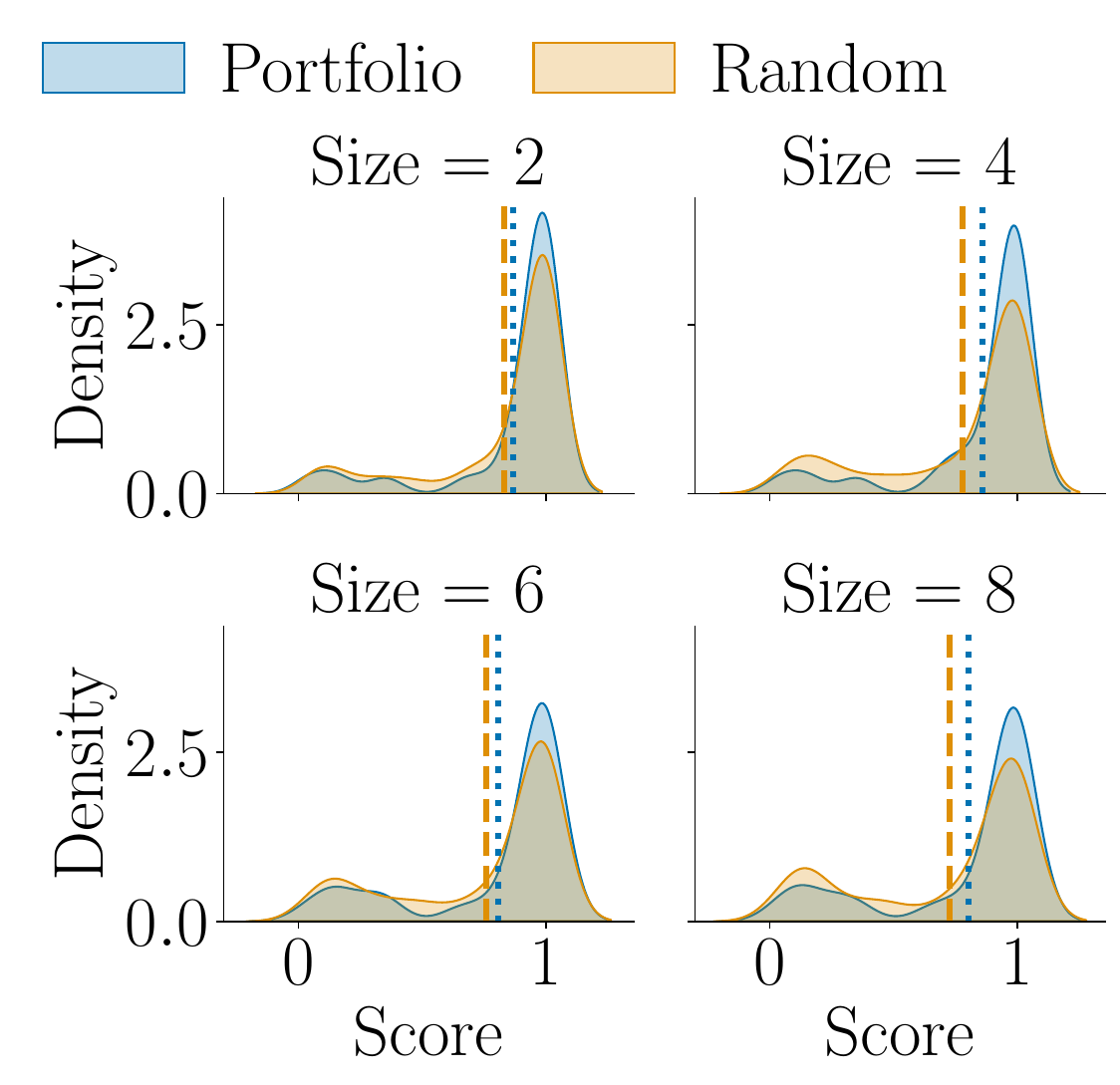}}
    \caption{Kde plot of scores assigned by \texttt{gpt-5.4}-as-a-judge to  portfolios of size $s \in \{2,4,6,8\}$ against $s$ randomly selected candidates for two evaluator types. The score distributions are over $25$ problems and $30$ samplings of the random candidates. Dashed and dotted lines represent the mean score value over the randomly chosen and the portfolio candidates respectively.}
    \label{fig:real}
\end{figure}

We construct portfolios of optimization models for $25$ optimization problems given in natural language from the dataset NL4LP~\citep{ahmaditeshnizi24optimus}. 
We use a strong LLM-as-a-judge to evaluate the quality of the candidate optimization models in our portfolios and compare it against portfolios constructed by random sampling (i.e., prompting multiple times).  
Refer to Appendix~\ref{app:details-real} for all LLM prompts and further details. 

\xhdr{Generator} For each optimization problem, we generate $50$ candidate optimization models with \texttt{gpt-5.4-nano}, where each model comprises a natural language description followed by the corresponding python code. 
For each candidate model $o$ we compute the generator probability $p(o)$ using the normalized log probabilities of the tokens of the text representation of the candidate model.

\xhdr{Evaluator} %
For each generated model, we execute the code and use the output---together with the model and problem---to assign a scalar score from 1 to 100 to each generated model using \texttt{gpt-5.4-nano} with repeated sampling. 
For each problem, we rank the generated models with the average score over $4$ samples and use this as the evaluator ranking.
Further, we implement an additional evaluator type, where we rank the generated models using the probability distribution induced by the generator---essentially using the generator as evaluator too.   

\xhdr{Results using LLM-as-a-judge} We use \texttt{gpt-5.4} as a judge under the  same scoring regime as for the evaluator for each generated model, except that in the scoring input prompt, we also add the ground truth solution of the optimization model as provided in the NL4LP dataset. 
As the final score of each generated model we consider the average normalized LLM-as-a-judge score over the repeated samplings. 
We characterize the quality of each portfolio with the minimum among the scores of the models it contains, assuming a worst-case scenario, where a human would select the worst model in the portfolio.   
%
%
Figure~\ref{fig:real} shows that our portfolio for both the evaluator types  consistently outperforms random portfolios under this worst-case scenario. Further, it also shows that our portfolios using the reasoning capabilities of the LLM evaluator achieve higher scores compared to portfolios using the generator probabilities. These results demonstrate in practice the competitive advantage of our approach on unifying the generative and reasoning capabilities of LLMs.


\section*{Acknowledgments}
Straitouri acknowledges support from a Google PhD Fellowship. Kim and Tambe acknowledge support by ONR MURI N00014-24-1-2742.



\section*{Impact Statement}


This paper presents work whose goal is to advance the field of Machine
Learning. There are many potential societal consequences of our work, none
which we feel must be specifically highlighted here.



\bibliography{robust-portfolio}
\bibliographystyle{icml2026}

\newpage
\appendix
\onecolumn
\section{Proofs}
\label{app:proofs}
\subsection{Proof of Corollary~\ref{cor:eval-robustness}}
\begin{proof}
    Let $\pi_{e}$ be the ranking policy of the aligned evaluator.
    By Definition~\ref{def:evaluator-alignment}, we have that $\pi_{e}(d) = \pi^*(d)$, thus $o_{(i)^e} = o_{(i)*}$ for any $i \in \NN$. 
    For any generator model and $\alpha \in (0,1)$ we use Eq.~\ref{eq:portfolio}, to construct a portfolio $\Pcal(d;\alpha) = \{ o_{(j)^e}\}_{i=1}^{k*(\alpha)}$, where $k^{*}(\alpha)$ is given by Eq.~\ref{eq:portfolio-size}. 
    The portfolio $\Pcal(d ; \alpha)$ has coverage 
    \begin{align*}
        c(\Pcal(d ; \alpha)) = c \left (\{ o_{(i)^e}\}_{i=1}^{k^{*}(\alpha)} \right ) = c \left (\{ o_{(i)^*}\}_{i=1}^{k^{*}(\alpha)} \right ) =  \frac{\sum_{i=1}^{k^{*}(\alpha)} \II \left \{o_{(i)^*} \in \{ o_{(i)^*}\}_{i=1}^{k^{*}(\alpha)} \right \}}{k^*(\alpha)} = 1.
    \end{align*}
\end{proof}

\subsection{Proof of Proposition~\ref{prop:coverage-bound}}
\begin{proof}
    With a slight abuse of notation, we will write $\Pcal$ to denote $\Pcal(d ; \alpha)$ for ease of exposition.
    Let $\Xcal \subseteq \Pcal$ such that $\Xcal = \Pcal \cap \Pcal^*$, where $\Pcal^*$ is a portfolio of  size $|\Pcal^*| = |\Pcal|$ with coverage $c(\Pcal^{*}) = 1$. For $\Xcal \subseteq \Pcal$ we have 
    \begin{align}~\label{eq:x-portfolio-rule}
        \sum_{o \in \Xcal}p(o) + \sum_{o' \in \Pcal \setminus \Xcal}p(o') \geq 1 - \alpha.
    \end{align}
    Since the generator model is human aligned, it must hold that $\forall o \in \Pcal\setminus \Xcal$ and $\forall o' \in \Pcal^{*}\setminus \Xcal$, $rank(o) \geq rank(o')$ given the human ranking $\pi^{*}(d)$ and as a result, $p(o) \leq p(o')$. Using this and the fact that $|\Pcal\setminus \Xcal| = |\Pcal^{*}\setminus \Xcal| = k^{*}(\alpha) - |\Xcal|$, Eq.~\ref{eq:x-portfolio-rule} becomes
    \begin{align}~\label{eq:x-f-star-alpha}
         \sum_{o \in \Xcal}p(o) + \sum_{o'' \in \Pcal^{*} \setminus \Xcal}p(o'') \geq \sum_{o \in \Xcal}p(o) + \sum_{o' \in \Pcal \setminus \Xcal}p(o') \geq 1 - \alpha.
    \end{align}
    
    We will now use that $\sum_{o'' \in \Pcal^{*} \setminus \Xcal}p(o'') \leq \alpha$. This is because $\Pcal^{*} \setminus \Xcal \subseteq \Ocal \setminus \Pcal$ and by Eq.~\ref{eq:x-portfolio-rule}
    \begin{align*}
        1 - \sum_{o \in \Ocal\setminus\Pcal}p(o) = 
        \sum_{o \in \Xcal}p(o) + \sum_{o' \in \Pcal \setminus \Xcal}p(o') \geq 1 - \alpha.
    \end{align*}
    As a result,
    \begin{align*}
        \sum_{o \in \Ocal\setminus\Pcal}p(o) \leq \alpha
    \end{align*}
    and since  $\Pcal^{*} \setminus \Xcal \subseteq \Ocal \setminus \Pcal$ it also holds 
    \begin{align*}
        \sum_{o'' \in \Pcal^{*} \setminus \Xcal}p(o'') \leq \sum_{o \in \Ocal\setminus\Pcal}p(o) \leq \alpha.
    \end{align*}
    Using the above and Eq.~\ref{eq:x-f-star-alpha}, we have 
    \begin{align*}
          \sum_{o \in \Xcal}p(o) + \alpha \geq \sum_{o \in \Xcal}p(o) + \sum_{o'' \in \Pcal^{*} \setminus \Xcal}p(o'') \geq 1 - \alpha
    \end{align*}
    and therefore
    \begin{align*}
        \sum_{o \in \Xcal}p(o) \geq 1 - 2\alpha.
    \end{align*}
    We will now use that $\sum_{o \in \Xcal}p(o) < |\Xcal| = k^{*}(\alpha) \cdot \frac{|\Xcal|}{k^{*}(\alpha)}$ to rewrite the above as 
    \begin{align*}
        k^{*}(\alpha) \cdot \frac{|\Xcal|}{k^{*}(\alpha)} > \sum_{o \in \Xcal}p(o) \geq 1 - 2\alpha.
    \end{align*}  
    Remember that $\Xcal = \Pcal \cap \Pcal^{*}$, so $c(\Pcal) = \frac{|\Xcal|}{k^{*}(\alpha)}$. Therefore the above becomes
    \begin{align*}
         k^{*}(\alpha) \cdot c(\Pcal) > 1 - 2\alpha
    \end{align*}
    and, thus
    \begin{align*}
        c(\Pcal) > \frac{1 - 2\alpha}{k^{*}(\alpha)}.
    \end{align*}
\end{proof}

\section{Additional Experimental Details}
\label{app:details}
\subsection{Simulated Portfolios}
\xhdr{Generator Types}
\label{app:deatils-simulation}
\begin{itemize}
\item \textsc{Aligned}. Each generated candidate $i \in [K]$ has probability $p(i) = \frac{K+1-i}{\sum_{i'=1}^K K+1 - i'} $; therefore $\forall i < j \in [K], p(i) > p(j)$ satisfying Definition~\ref{def:generator-alignment}.
\item \textsc{Weakly Aligned}.  Each generated candidate $i \in [K]$ has a randomly selected probability value such that $\sum_{i \in [K/2]} p(i) > \sum_{j \in [K/2]} p(K/2 + j)$, where we select $\sum_{i \in [K/2]} p(i) \in [0.5, 0.99)$ uniformly at random. Note that Proposition~\ref{prop:coverage-bound} holds for any portfolio of $K/2$ generation candidates constructed using this generator type.
\item \textsc{Uniform}. Each generated candidate $i \in [K]$ has probability $p(i) = 1/K$.
\item \textsc{Misaligned}. Each generated candidate $i \in [K]$ has probability $p(i) = \frac{i}{\sum_{i'=1}^K i'}$.
\end{itemize}

\subsection{Portfolios of Optimization Models}
\label{app:details-real}
We prompt all LLMs with temperature $1.0$, while setting verbosity to `\texttt{low}' and the reasoning effort to `\texttt{none}'. Below we provide the prompt templates and system prompts we used for the generator, evaluator and LLM-as-a-judge models. For the templates we follow the templates used by~\citet{huang2025orlm}. 

\clearpage
\newpage
\begin{figure}[t]
    \centering
    \includegraphics[width=\linewidth]{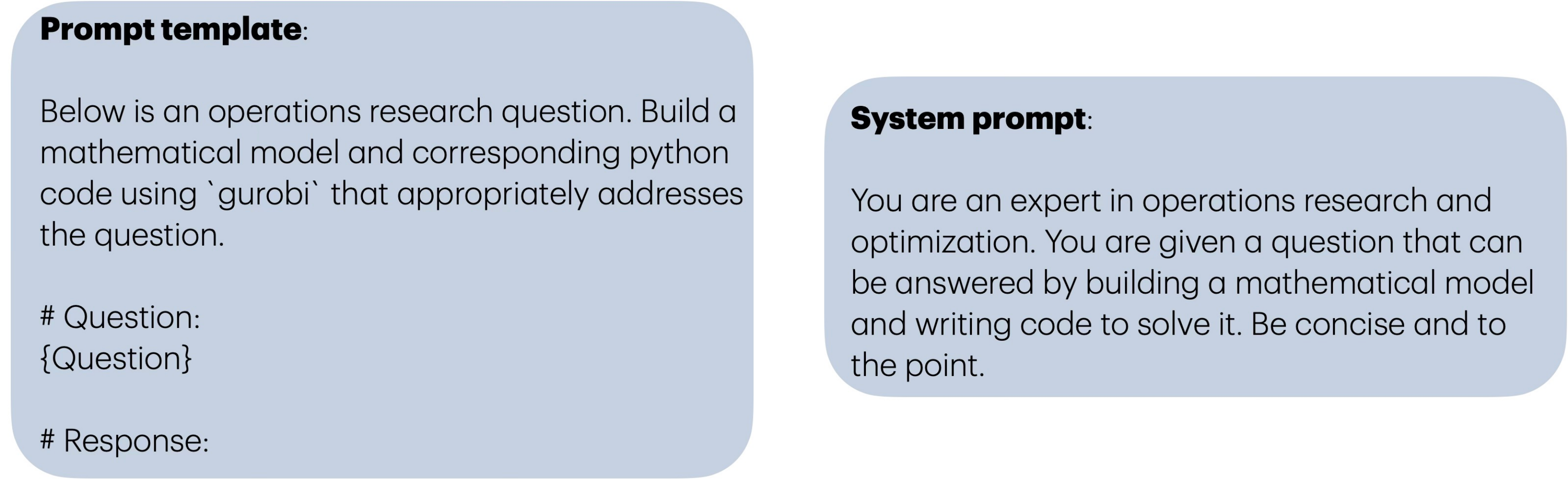}
    \caption{Generator prompt and system prompt.}
    \label{fig:generator-prompt}
\end{figure}

\begin{figure}[h!]
    \centering
    \includegraphics[width=\linewidth]{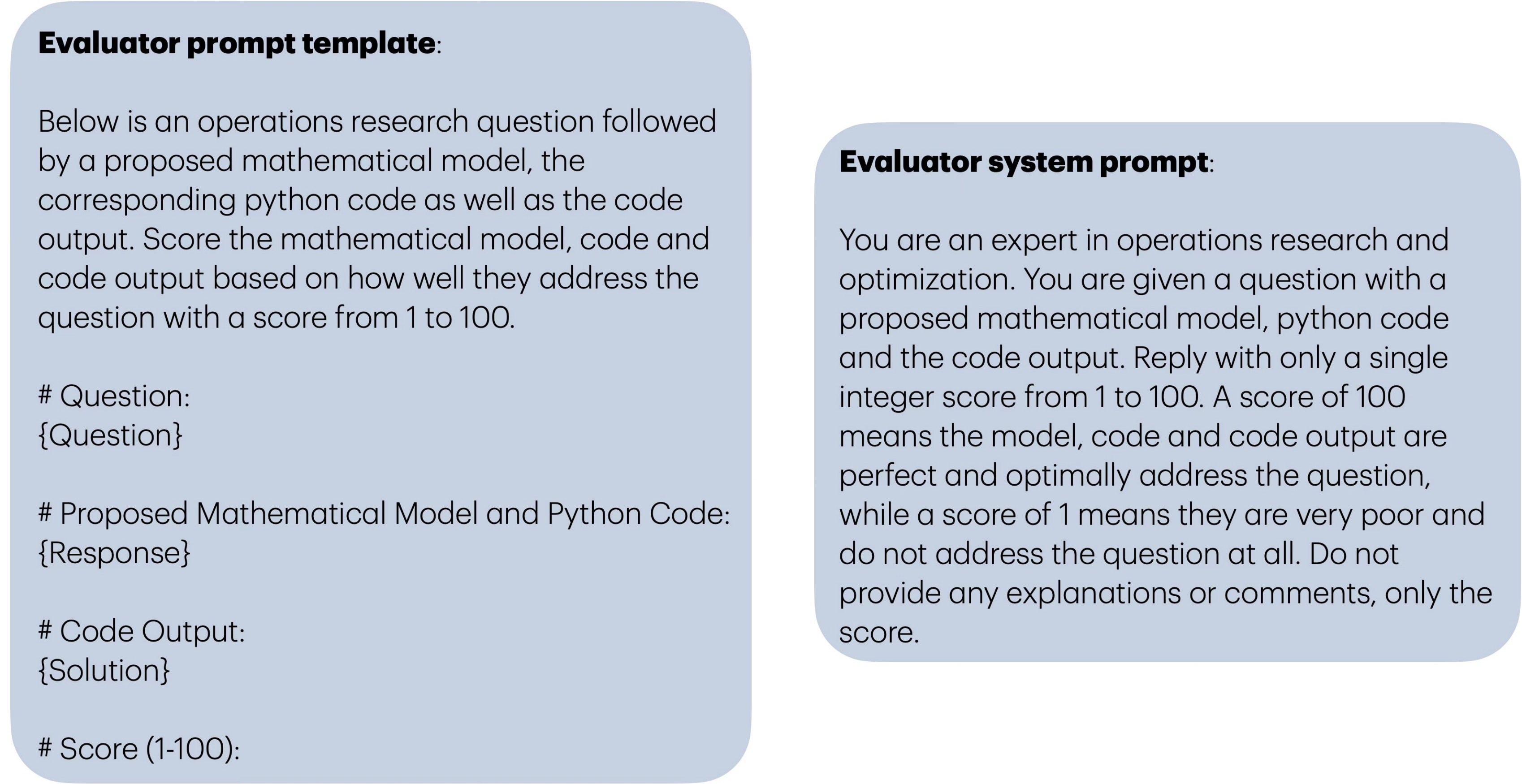}
    \caption{Evaluator prompt and system prompt.}
    \label{fig:evaluator-prompt}
\end{figure}
\begin{figure}[t]
    \centering
    \includegraphics[width=\linewidth]{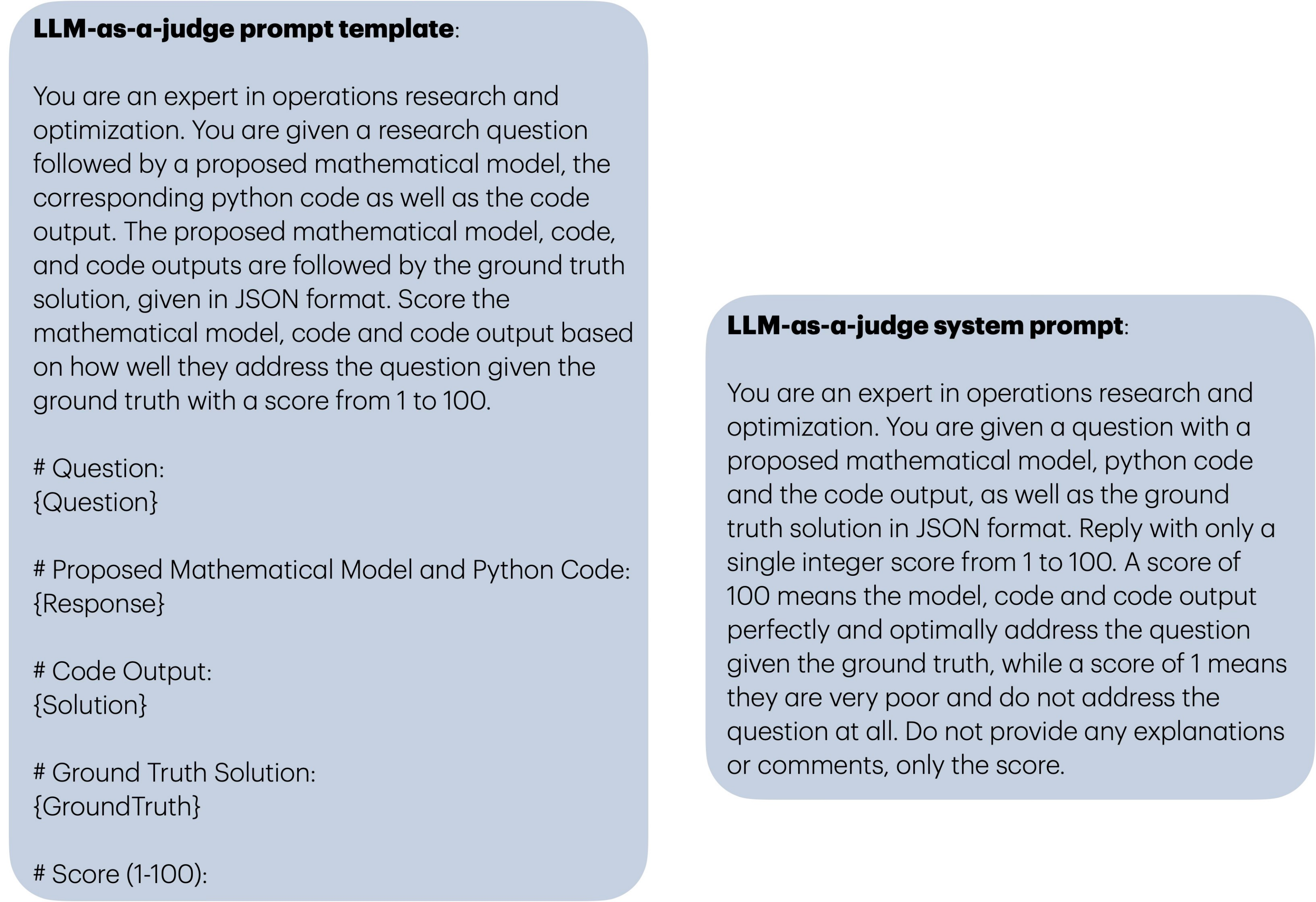}
    \caption{LLM-as-a-judge prompt and system prompt.}
    \label{fig:judge-prompt}
\end{figure}
\clearpage
\newpage
\section{Additional Experimental Results on Simulated Portfolios}
\label{app:synthetic}
Below we present results over several generator-evaluator pairs and values of $K$. 
\begin{figure}[h!]
    \centering
    \includegraphics[width=\linewidth]{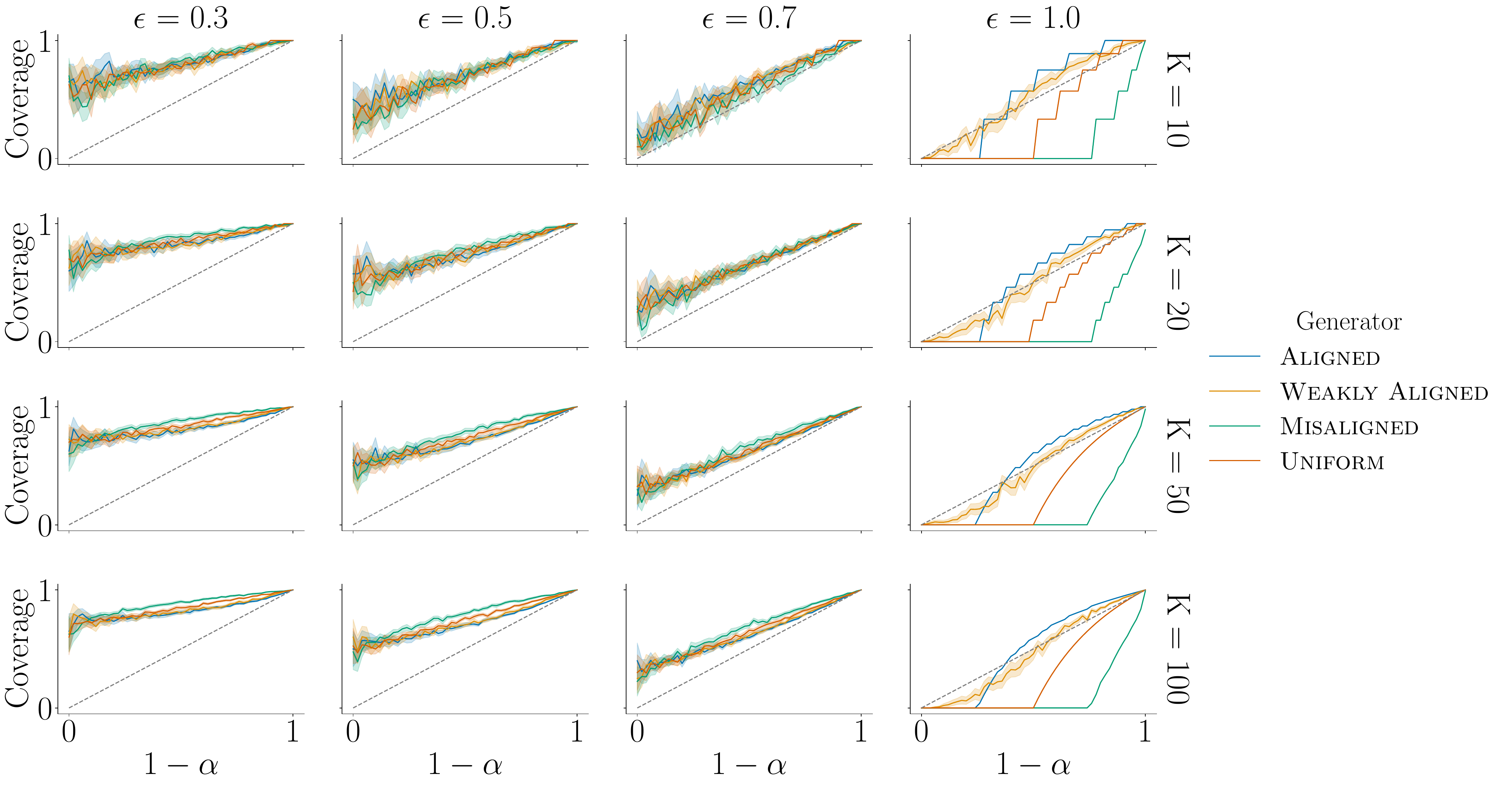}
    \caption{Portfolio coverage against $1-\alpha$ for several generator evaluator pairs and value of $K$. The mean is over $40$ iterations and shaded areas represent $95\%$ confidence intervals.}
    \label{fig:grid-coverage}
\end{figure}

\begin{figure}[h!]
    \centering
    \includegraphics[width=\linewidth]{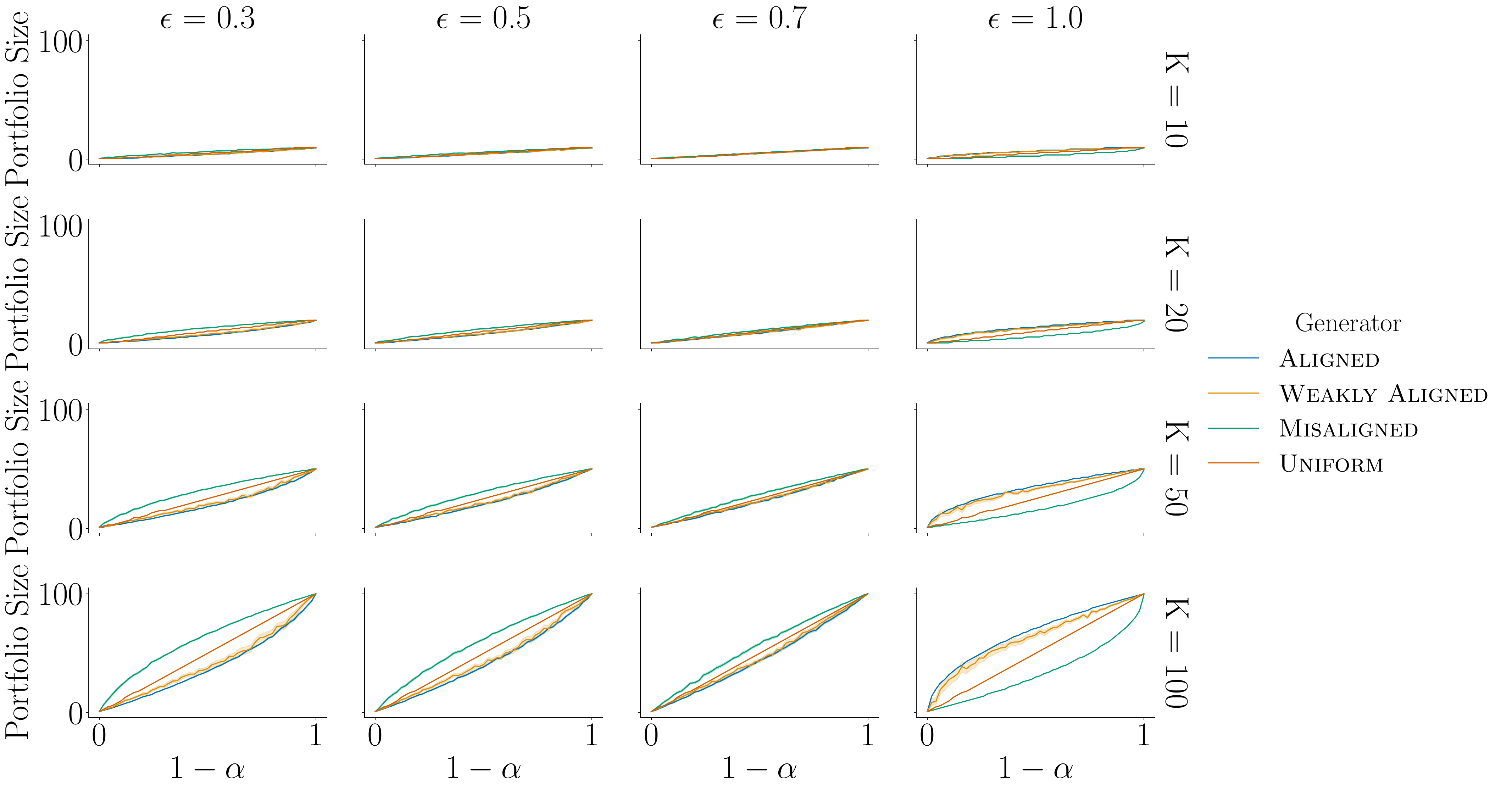}
    \caption{Portfolio size against $1-\alpha$ for several generator evaluator pairs and value of $K$. The mean is over $40$ iterations and shaded areas represent $95\%$ confidence intervals.}
    \label{fig:grid-size}
\end{figure}


\end{document}